\documentclass[
	10pt, 
	two column, 
	twoside
]{IEEEtran}
\usepackage{amssymb,amsmath,mathrsfs}
	\newcommand{\bs}{\boldsymbol}
\usepackage[ 
	lined,
	boxed,
	commentsnumbered, 
	ruled
]{algorithm2e}
\usepackage{stfloats}
\usepackage{subfigure}
\usepackage{graphicx,booktabs,multirow}
\usepackage[noadjust]{cite}
\usepackage{dsfont}
\usepackage{bbding}
\usepackage{pifont}
\usepackage{framed}
\usepackage{soul}
\usepackage{subfigure}
\usepackage{fourier-orns}
\usepackage{graphicx,booktabs,multirow,cuted,mathtools}
\usepackage{bm}
\usepackage{amsthm}
	\theoremstyle{definition}
		\newtheorem{lemma}{Lemma}
		\newtheorem{assumption}{Assumption}
		\newtheorem{theorem}{Theorem}

		\newtheorem{remark}{Remark}
\usepackage{calligra}
\usepackage{pgfplots}

	\definecolor{colorhkust}{HTML}{142B8C}
	\definecolor{colorshanghaitech}{HTML}{A20005}
	\definecolor{colortsinghua}{HTML}{743481}
	\definecolor{colordark}{RGB}{184,134,11}
	\definecolor{colorRed}{RGB}{128, 0, 0}
	\definecolor{colorGreen}{RGB}{0, 64, 0}
	\definecolor{colorBlue}{RGB}{0, 0, 128}
\usepackage[
	colorlinks, 
	linkcolor=colorRed, 
	anchorcolor=blue, 
	citecolor=colorBlue, 
	urlcolor = colorGreen
]{hyperref}

\begin{document}

\title{Federated Dropout: Convergence Analysis and Resource Allocation}

\author
{Sijing Xie, Dingzhu Wen, Xiaonan Liu, Changsheng You, Tharmalingam Ratnarajah, Kaibin Huang 

    \thanks{S. Xie, and D. Wen are with the School of Information Science and Technology, ShanghaiTech University, Shanghai 201210, China (e-mail: \{xiesj2023, wendzh\}@shanghaitech.edu.cn). 
    }
    \thanks{X. Liu is with the School of Natural and Computing Science, University of Aberdeen, AB24 3FX Aberdeen, United Kingdom (email: xiaonan.liu@abdn.ac.uk).}
    \thanks{C. You is with the Department of Electronic and Electrical Engineering, Southern University of Science and Technology, Shenzhen 518055, China (e-mail: youcs@sustech.edu.cn).}
    \thanks{T. Ratnarajah is with the School of Engineering, The University of Edinburgh, EH8 9YL Edinburgh, United Kingdom (e-mail: t.ratnarajah@ed.ac.uk).}
    \thanks{K. Huang is with the Department of Electrical and Electronic Engineering, The University of Hong Kong, Hong Kong (e-mail: huangkb@eee.hku.hk).}

}

\maketitle
\IEEEpeerreviewmaketitle

\begin{abstract}
Federated Dropout is an efficient technique to overcome both communication and computation bottlenecks for deploying federated learning at the network edge. In each training round, an edge device only needs to update and transmit a sub-model, which is generated by the typical method of dropout in deep learning, and thus effectively reduces the per-round latency. \textcolor{blue}{However, the theoretical convergence analysis for Federated Dropout is still lacking in the literature, particularly regarding the quantitative influence of dropout rate on convergence}. To address this issue, by using the Taylor expansion method, we mathematically show that the gradient variance increases with a scaling factor of $\gamma/(1-\gamma)$, with $\gamma \in [0, \theta)$ denoting the dropout rate and $\theta$ being the maximum dropout rate ensuring the loss function reduction. Based on the above approximation, we provide the convergence analysis for Federated Dropout. Specifically, it is shown that a larger dropout rate of each device leads to a slower convergence rate. This provides a theoretical foundation for reducing the convergence latency by making a tradeoff between the per-round latency and the overall rounds till convergence. Moreover, a low-complexity algorithm is proposed to jointly optimize the dropout rate and the bandwidth allocation for minimizing the loss function in all rounds under a given per-round latency and limited network resources. 
Finally, numerical results are provided to verify the effectiveness of the proposed algorithm.
\end{abstract}

\begin{IEEEkeywords}
Federated edge learning, dropout, convergence analysis, communication and computational efficiency.
\end{IEEEkeywords}

\section{Introduction}
The explosion of mobile data in wireless networks drove the deployment of \textit{artificial intelligence} (AI) at the network edge \cite{zhu2023pushing,li2022deep,wen2023task, 10579852}. 
The integration of communication and AI has been one of the six usage scenarios of 6G by IMT-2030 \cite{IMT2030}. \textit{Federated edge learning} (FEEL) has emerged as a promising technique to realize edge AI by enabling distributed model training over edge devices while preserving data privacy \cite{shi2023task, Xu2023adaptive}.

In the existing literature, several categories of techniques have been proposed to facilitate the communication and energy efficient implementation of FEEL. Among the rest, the first technique aims at overcoming the server communication bottleneck caused by transmitting large-scale AI models, especially \textit{deep neural networks} (DNNs) \textcolor{blue}{and \textit{large language models} (LLMs)}. Typical communication-efficient FEEL schemes include \textcolor{blue}{partial gradient averaging \cite{10304494},} gradient or model compression \cite{wu2023fedcomp,xu2022adaptive}, resource management \cite{jiang2020joint,wadu2021joint,lim2021decentralized}, device scheduling \cite{ren2020scheduling,yang2019scheduling}, over-the-air computation \cite{cao2021optimized,mu2022federated,tegin2023federated}, and knowledge distillation \cite{10474173,liu2022communication}, etc. The second category of techniques aims to reduce the energy consumption of edge devices over multiple rounds via e.g., energy harvesting \cite{zeng2023federated}, and joint resource management via CPU-GPU computing coordination \cite{zeng2021energy}. Moreover, several new frameworks of FEEL have been proposed, e.g., decentralized FEEL \cite{liu2023communication,yang2023decentralized}, hierarchical FEEL \cite{chen2023enhanced,mhaisen2021optimal}\textcolor{blue}{, and personalized \textit{Federated Learning} (FL) \cite{setayesh2023perfedmask} }.

\textcolor{blue}{However, edge devices continue to face significant limitations in computational resources \cite{jiang2023computation, mcmahan2017communication}. Concurrently, the increasingly complex reality tasks arouse deep concern on huge parameter size of DNNs or LLMs, demanding substantial computational power. For instance, a ResNet-18 has more than 11 million parameters and entails more than 1.8G float-point operations (FLOPs). Such resource demands are often unsustainable for constrained platforms like embedded edge devices \cite{DBLP:conf/iclr/0022KDSG17}. This results in suboptimal model performance and unbearable time delay, hindering the swift deployment and practical application of AI models. Unfortunately, the above techniques are dedicated to enhancing communication and energy efficiency and do not essentially reduce the model complexity. As a consequence, the computational overhead is still high, as these methods necessitate updating the entire large-scale AI model during each training round, which may become the primary bottleneck of these approaches.} To tackle these challenges, model-pruning-based FL was proposed, e.g., PruneFL \cite{jiang2022model} and adaptive model pruning for hierarchical FEEL \cite{liu2023adaptive, 10678894}.
However, the communication overhead is still very large at the early stage of training. Moreover, this method only retains a sub-model after training, thus generally suffering from lower representation ability than the original model.  

To address the issues in model-pruning-based FL, a federated dropout (FedDrop) framework was proposed in \cite{FedDrop}. This framework made the first attempt to utilize the typical technique of dropout in deep learning \cite{srivastava2014dropout}. It generates a subset for each device during each round. \textcolor{blue}{FedDrop alleviates computational overhead by reducing the number of parameters and FLOPs, while also mitigating communication overhead by decreasing the size of data transmissions, making it particularly efficient in resource-constrained environments. Although different subnets are updated in each round of training, the original network structure remains intact after model convergence, thus maintaining the performance and functionality. Additionally, FedDrop provides a highly flexible and adjustable dropout rate to optimize performance across various applications. Moreover, FedDrop effectively mitigates the risk of overfitting by serving as a robust regularization technique, which enhances the model's generalization capabilities. Lastly, FedDrop can seamlessly integrate with other algorithms, e.g., split learning \cite{thapa2022splitfed}, to further improve learning performance while maintaining overall system efficiency.}

In addition, other types of subnet training frameworks include static-based \cite{horvath2021fjord}, rolling-based \cite{alam2022fedrolex}, and importance-based subnet generation schemes \cite{feng2023feddd,liu2021adaptive}. In the static-based scheme, each sub-model is always extracted from a designated part and remains the same over different rounds. Particularly, in \cite{horvath2021fjord}, the authors proposed to drop adjacent components of a neural network, since the order of parameters was important and should be kept. However, different parts of the global model are trained with different data distributions, thus suffering from degraded training quality. Further, to cover the whole global model, at least one device needs to train the whole model, so that the size of the global model is subject to the device with the largest model. A rolling sub-model extraction was introduced in \cite{alam2022fedrolex}, which enabled different parts of the model to be trained evenly with a rolling method. However, to finish a rolling window, it requires several communication rounds, with the same dropout rate in these rounds, thus failing to adapt to the time-varying communication and computation environments. On the other hand, importance-based schemes were proposed in \cite{feng2023feddd,liu2021adaptive} by characterizing the effects of different parameters in convergence or the model error after removing it. In particular, in \cite{liu2021adaptive}, the authors simply assumed that a subnet's gradient vector was an unbiased estimation of the original network, which however, has not been theoretically justified. \textcolor{blue}{To enhance generalization on non-independent and identically distributed (non-IID) datasets, \cite{tenisongradient} proposed a method analyzing the consistency of model update directions and employing a soft masking operation, which can be adapted as a drop-in replacement. However, it only concludes that the convergence rate is affected by the masking computation method, learning rate, and data distribution qualitatively.} Besides, in importance-based schemes, extra computation latency is incurred to determine the importance level of each weight. In summary, compared with other subnet training frameworks, FedDrop enjoys the following benefits. 
 \begin{itemize}
     \item {\bf Low design complexity:} FedDrop is simple in practical implementation, which requires no additional operations except dropout, and thus has low design and computational complexity.
     \item {\bf Versatile adaptiveness:} The dropout rate can be properly designed to adapt to devices' computational and communication capabilities, network resources, and wireless channel conditions.
     \item {\bf High model diversity:} FedDrop enjoys a high model diversity among different rounds and different devices due to randomness, namely, training diverse neuron groups in diverse architectures \cite{srivastava2014dropout}.
     \item {\bf Strong model robustness:} FedDrop increases the robustness of the model and removes any simple dependencies between the neurons.
 \end{itemize}
\textcolor{blue}{Despite the above advantages, the important and rigorous theoretical convergence analysis for FedDrop remains underexplored in the existing literature, which thus motivates the current work. The convergence proof of FedDrop can be generalized and extended to other algorithms as well.}

In this paper, we consider a FedDrop enabled wireless FL system, where an edge server is deployed to cooperatively train the FL model with multiple edge devices. We first provide the theoretical convergence analysis for FedDrop and then propose an efficient algorithm to jointly optimize the dropout rate and bandwidth allocation for efficient FedDrop. 
The main contributions are elaborated as follows.
\begin{itemize}
    \item {\bf Convergence analysis}: \textcolor{blue}{Based on the randomness,} by using the Taylor expansion of a subnet's gradient vector and not considering wireless channels in the analysis, we show that the gradient vector of a subnet generated by dropout can be approximated as a variance-bounded estimation of the original DNN's gradient vector. The gradient variance of a subnet is proportional to $\gamma/(1-\gamma)$ with $\gamma \in [0, \theta)$ being the dropout rate and $\theta$ being the maximum dropout rate ensuring the loss function reduction. Based on this approximation, theoretical analysis is conducted. It shows that dropout rates cannot be too large to guarantee that the loss function decreases after each training round. Furthermore, the convergence rate turns out to be slower with a higher dropout rate. 
    This provides a theoretical foundation to make a tradeoff between the per-round computation-and-communication efficiency and the overall rounds till convergence for minimizing the convergence latency. 
    
    
    \item {\bf Per-round loss function minimization}: Based on the theoretical convergence analysis for FedDrop, we characterize the learning loss reduction of an arbitrary round, which is shown to be dependent on the dropout rate of each device. Particularly, a higher dropout rate will lead to less decrease in the loss function. By leveraging the fact that both communication and computational overhead of a DNN is reduced to a ratio of $(1-\gamma)$ with a dropout of $\gamma$, we formulate an optimization problem to maximize the convergence rate by maximizing the learning loss reduction in each round. As the exact form of the learning loss reduction is intractable, its lower bound is maximized instead without loss of generality, under the constraints of limited system bandwidth, task completion latency, and devices' energy budgets. This design target compromises the per-round latency and the required number of rounds till convergence, as well as involving more resource-limited devices in FEEL for contributing their computational and data resources.
    \item {\bf Joint design of adaptive dropout rate and bandwidth allocation}: \textcolor{blue}{We convert the per-round loss function minimization problem to be convex and use Karush-Kuhn-Tucker (KKT) conditions to obtain closed-form solutions of dropout rate and bandwidth allocation.} The device with better uplink and downlink channel conditions should be allocated with a lower dropout rate and a smaller bandwidth. Moreover, the device’s bandwidth allocation ratio for each device \textcolor{blue}{shows the opposite trend} to its dropout rate. Based on these results, an optimal joint design of the dropout rate and bandwidth allocation is proposed with a low complexity of $O(K^2)$, where $K$ is the total number of devices.
    
    \item {\bf Performance evaluation}: To evaluate the performance of the proposed scheme, we conduct numerical experiments in two scenarios of underfitting and overfitting, \textcolor{blue}{corresponding to LeNet and AlexNet trained on the CIFAR-100 dataset.} The numerical results show that a smaller dropout rate leads to faster convergence, which verifies our theoretical analysis. In both scenarios, more network resources, i.e., a longer latency and a larger system bandwidth, allow lower dropout rates of all devices, leading to an improved testing performance. Besides, in the overfitting scenario, the proposed scheme outperforms the scheme without dropout, which is because dropout can avoid overfitting. 
    
    


\end{itemize}
The rest of the paper is organized as follows. Section II introduces the system model. The convergence analysis and problem formulation are presented in Section III, followed by the joint optimization of dropout rate and resource allocation in Section IV. In Section V, the simulation results are provided. Section VI concludes the paper.

\section{System Model}\label{sec:SystModel}
\subsection{Network Model}\label{subsec:netModel} 
A single-cell network is used for implementing the FedDrop framework, as illustrated in Fig. \ref{fig:network}. In the system, one edge server equipped with a single-antenna access point (AP) and $K$ single-antenna edge devices, denoted by $\mathcal{K}=\{1, 2,..., K \}$, cooperatively complete a FL via wireless links. In each round, each device first downloads a subnet from the server, then updates it using the local dataset by stochastic gradient descent (SGD), and finally uploads the updated subnet back to the server. The server collects all updated subnets for updating the global model. Without loss of generality, the channels are assumed to be frequency non-selective (see, e.g., \cite{wen2020joint}). Besides, the channels remain static in each round and vary over different rounds. The system bandwidth is denoted as $B$, which is divided into $K$ orthogonal sub-bands in each round. Similar to the Orthogonal Frequency Division Multiplexing (OFDM) system, different users are given different numbers of subcarriers to make sure there are no intra-cell interferences. Each sub-band is assigned to one device for downloading and uploading. The server is assumed to have the channel state information (CSI) of all devices' downlinks and uplinks by using effective channel estimation methods.
\begin{figure*}[hbtp]
    \centering   \includegraphics[width=0.8\linewidth]{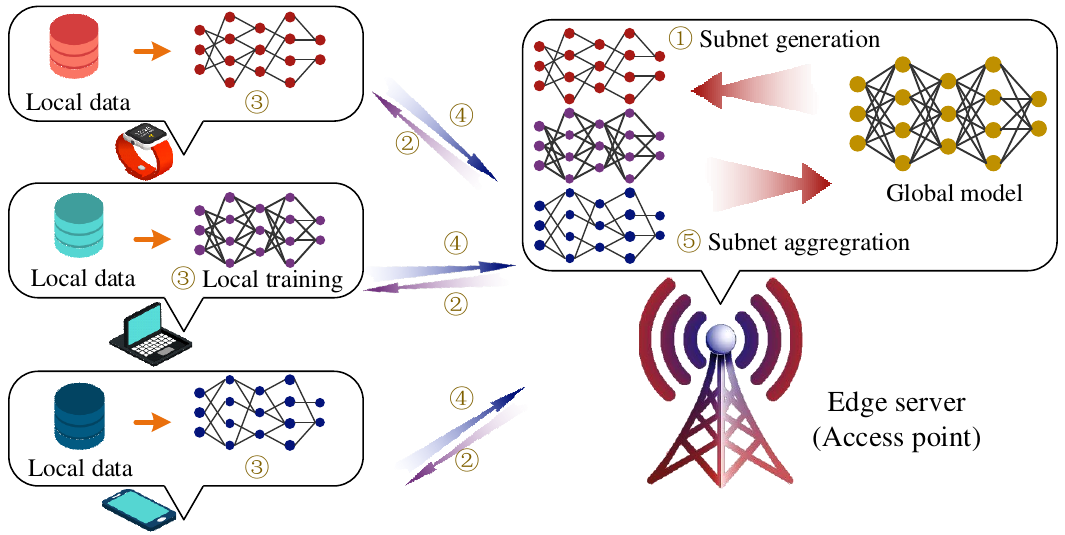}
    \caption{The operations of FL with FedDrop in a wireless system.}
    \label{fig:network}
\end{figure*}

\subsection{Federated Learning Model}
In the FedDrop system, each device $k$ holds a local dataset, denoted by $\mathcal{D}_k = \{\boldsymbol{x}_i | i=1,2,\ldots,|\mathcal{D}_k| \}$, where $\boldsymbol{x}_i$ is the $i$-th data sample and the size of $\mathcal{D}_k$ is $|\mathcal{D}_k|$. Then, the objective of the FL task is to minimize the global loss function:
\begin{equation}
    \mathcal{F}(\boldsymbol{w})=\sum_{k=1}^K \frac{\left|\mathcal{D}_k\right|}{|\mathcal{D}|} f_k\left(\boldsymbol{\hat{w}}_k ; \mathcal{D}_k\right),
\end{equation}
where $\mathcal{D} = \{\mathcal{D}_k\}$ is the global dataset, $\boldsymbol{\hat{w}}_k$ is the dropout generated subnet associated with the $k$-th device, $f_k\left(\boldsymbol{\hat{w}}_k; \mathcal{D}_k\right)$ is the local loss function of the $k$-th device given by
\begin{equation}
   f_k\left(\boldsymbol{\hat{w}}_k ; \mathcal{D}_k\right)=\frac{1}{|\mathcal{D}_k|}\sum_{\boldsymbol{x}_i\in\mathcal{D}_k}f\left(\boldsymbol{\hat{w}}_k ; \boldsymbol{x}_i\right),
\end{equation}
where $f(\cdot;\cdot)$ is the empirical loss function that characterizes the difference between the model output and real label via e.g., the cross entropy and mean square error (MSE). 

\subsection{Dropout}\label{Sect:Dropout}
We adopt the dropout technique for random subnet creation. The original dropout technique was proposed in \cite{srivastava2014dropout} to avoid overfitting of DNNs during training and could be applied to the weights, as well as the units \cite{baldi2013understanding}. \textcolor{blue}{In this paper, we focus exclusively on applying dropout to the weights of all layers, but our results also hold true for dropout applied to the units, with some adjustments.} In each round of FedDrop, the server generates a subnet for each device using dropout. Specifically, for device $k$ with a dropout rate of $\gamma_k \in [0,1)$, the dropout technique deactivates the weights of a DNN with a probability of $\gamma_k$. After the random dropout, the subnet is given by
\begin{equation}
    \begin{aligned}
       \hat{\boldsymbol{w}}_{k}  &=   \boldsymbol{w} \circ \boldsymbol{m}_{k} ,\;\forall k,   
    \end{aligned}
\end{equation}
where $\boldsymbol{w}$ is the original parameter vector without dropout, $\circ$ represents element-wise multiplication, and $\boldsymbol{m}_{k}$ is the Bernoulli dropout mask, the $j$-th element $m_{k,j}$ is defined as
\begin{equation}
\label{Eq:DropoutMask}
    m_{k,j}=\left\{
    \begin{array}{lll}
     \frac{1}{1-\gamma_{k}}, & \text { with probability of }\left(1-\gamma_{k}\right),  \\
     0, & \text { with probability of } \gamma_{k}. &
    \end{array}\right.
\end{equation}
Therein, the scaling factor $ \frac{1}{1-\gamma_{k}}$ in \eqref{Eq:DropoutMask} is used to guarantee that the expectation of the subnet weight is equal to the original network. As a result, the model with dropout equals the original one \cite{srivastava2014dropout}, i.e.,
$\mathbb{E}_{\bm{m}_{k}}\left(\hat{\boldsymbol{w}}_{k} \right)=\bm{w}$.

\subsection{FedDrop Framework}\label{subsec:Fedroppro} 

As shown in Fig. \ref{fig:network}, there are five steps in each training round of FedDrop \cite{FedDrop}, which are elaborated in the sequel.

\subsubsection{Generation (Subnets Generation)}
The server adopts the dropout technique introduced in Section \ref{Sect:Dropout} to generate a subnet for each device. In practice, a \textit{progressive random parametric pruning} approach proposed in \cite{FedDrop} is adopted for subnet generation. Specifically, for an arbitrary layer, we randomly and uniformly choose a weight and deactivate it, and repeat this step until meeting the requirement. 


\subsubsection{Push (Model Downloading)} Each device downloads its corresponding subnet from the server using a pre-assigned orthogonal frequency band.

\subsubsection{Computation (Local Model Updating)}
All devices update their corresponding subnets based on their local datasets utilizing SGD.

\subsubsection{Pull (Local Model Uploading)}
Each device uploads the updated local subnet model to the server by using the allocated bandwidth.

\subsubsection{Aggregation (Global Aggregation and Updating)}
In this step, a complete network is first constructed for each local subnet by zero padding, i.e., the dropped weights in the complete network are filled with zero. Then, all complete networks are aggregated for updating the global network.

The above five steps iterate until model convergence or reach the maximum round of iteration.

\subsection{Latency and Energy Consumption Models}
Consider an arbitrary communication round and an arbitrary device, say the $t$-th round and $k$-th device. The latency and energy consumption of these five steps in this round are modeled as follows.

\subsubsection{Generation Step}
The generation time at the server is the same for all devices and is much lower compared to other steps. No energy is consumed by devices in this step. 

\subsubsection{Push Step}
The latency of device $k$ to download its assigned subnet is
\begin{equation}\label{eqn:communicationdl_time}
    T_{k,t}^{\rm {com,dl}}=\frac{M_{k,t} Q}{r_{k,t}^{{\rm dl}}},
\end{equation}
where $Q$ is the quantization bits used for one parameter, $r_{k,t}^{\rm dl}$ is the downlink data rate of device $k$. \textcolor{blue}{$M_{k,t}$ is the number of parameters of the subnet, given as $M_{k,t}=(1-\gamma_{k,t})M_{\rm ori}$, where $\gamma_{k,t}$ is the dropout rate in this round, and $M_{\rm{ori}}$ is the number of all the parameters in the original network, respectively. We consider dropout in all the layers. Besides, the downlink data rate $r_{k,t}^{\rm dl}$ is}
\begin{equation}\label{equ:r_downlink}
    \begin{aligned}
    r_{k,t}^{\rm dl}
    &= \rho_{k,t}BR_{k,t}^{\rm dl}=\rho_{k,t}B\log_2 \left(1+\frac{|h_{k,t}^{\rm dl}|^2 P_k^{\rm{{com,dl}}}}{N_0} \right),
    \end{aligned}
\end{equation}
where $B$ is the total bandwidth, $\rho_{k,t}$ is the ratio of bandwidth assigned to device $k$, $R_{k,t}^{\rm dl}$ is the downlink spectrum efficiency, $N_0$ is the channel noise variance, $h_{k,t}^{\rm dl}$ is the downlink channel gain, and $P_k^{\rm{com,dl}}$ is the downlink transmit power of device $k$. The energy consumption in this step is to receive the model, which is included in the circuit energy consumption $\xi_k$.

\subsubsection{Computation Step}
In this step, the latency of device $k$ is given by 
\begin{equation}\label{eqn:computation_time}
    T_{k,t}^{\rm {cmp }}=\frac{C_{k,t}|\mathcal{D}_k|}{f_{k,t}},
\end{equation}
where $f_{k,t}$ (in cycle/s) is the CPU frequency of device $k$. \textcolor{blue}{$C_{k,t}$ is the number of processor operations to update the subnet using one data sample, given as $C_{k,t}=(1-\gamma_{k,t})C_{\rm {ori}}$, where $C_{\rm {ori}}$ is the number of processor operations for all the layers to update the original network.} Then, following \cite{you2017energy}, the energy consumption of device $k$ depends on $f_{k,t}$ and its latency $T_{k,t}^{\rm {cmp}}$, given by
\begin{equation}\label{eqn:computation_energy}
     E_{k,t}^{\rm {cmp}}=\Omega_kT_{k,t}^{\rm {cmp }} (f_{k,t})^3,
\end{equation}
where $\Omega_k$ is a constant characterizing the local computation performance of the processor on device $k$. 

\subsubsection{Pull Step}
The uploading latency of device $k$ is given as 
\begin{equation}\label{eqn:communicationul_time}
    T_{k,t}^{\rm {com,ul}}=\frac{M_{k,t} Q}{r_{k,t}^{\rm ul}},
\end{equation}
\textcolor{blue}{where $r_{k,t}^{\rm ul}$ is the uplink data rate given by}
\begin{equation} \label{equ:r_k,t}
    r_{k,t}^{\rm ul}= \rho_{k,t}B\log_2 \left(1+\frac{|h_{k,t}^{\rm ul}|^2P_{k,t}^{\rm{com,ul}}}{N_0} \right).
\end{equation}
In \eqref{equ:r_k,t}, \textcolor{blue}{$h_{k,t}^{\rm ul}$ is the uplink channel gain},  and other notations follow that in \eqref{equ:r_downlink}. 
Based on \eqref{equ:r_k,t}, the uplink transmit power is derived as
\begin{equation}\label{Eq:TransmitPower}
    P_{k,t}^{\rm{com,ul}}=\left(2^{\frac{r_{k,t}^{\rm ul}}{\rho_{k,t}B}}-1 \right)\frac{N_0}{|h_{k,t}^{\rm ul}|^2}.
\end{equation}
Then, the uploading energy consumption occurred in this step is given by 
\begin{equation} \label{eqn:communicationul_energy}
     \begin{aligned}
         E_{k,t}^{\rm{com,ul}}
         & =P_{k,t}^{\rm {com,ul}}T_{k,t}^{\rm {com,ul}}.
     \end{aligned}
\end{equation}

\subsubsection{Aggregation Step}
The global model aggregation time is the same for all devices and has no effect on the design, and is thus ignored. Besides, no energy is consumed by devices in this step.


In summary, the overall latency of device $k$ in this round is derived as
\begin{equation}\label{Eq:DeviceLatency}
    T_{k,t}=T_{k,t}^{\rm {com,dl}}+T_{k,t}^{\rm {cmp}}+T_{k,t}^{\rm {com,ul}},\forall k,
\end{equation}
where $T_{k,t}^{\rm {com,dl}}$ is the downlink communication latency defined in \eqref{eqn:communicationdl_time},
$T_{k,t}^{\rm {cmp}}$ is the computation latency defined in \eqref{eqn:computation_time}, and $T_{k,t}^{\rm {com,ul}}$ is the uplink communication latency defined in \eqref{eqn:communicationul_time}.\\
The total energy consumption of device $k$ in this round is 
\begin{equation}\label{Eq:DeviceEnergy}
    E_{k,t}=E_{k,t}^{\rm {com,ul}}+E_{k,t}^{\rm{cmp}}+\xi_k,\forall k,
\end{equation}
where $E_{k,t}^{\rm{cmp}}$ is the computation energy consumption in \eqref{eqn:computation_energy}, $E_{k,t}^{\rm{com,ul}}$ is the communication energy consumption in \eqref{eqn:communicationul_energy}, and $\xi_k$ is the circuit energy consumption for global model reception.

\textcolor{blue}{
\begin{remark}[Communication and Computational Overhead]
\textbf{Communication overhead:} FedDrop has an average upload and download cost of $32(1-\gamma_{k,t})M_{\rm ori}$ bits per client per round. PruneFL \cite{jiang2022model} requires clients to send full gradients to the server every $\Delta R$ rounds, leading to an average upload cost of $32(1-\gamma_{k,t})M_{\rm{ori}} +\frac{32}{\Delta R}M_{\rm{ori}} $. FedDST \cite{bibikar2022federated} also needs to upload and download masks if read-adjustments are executed, which once every $\Delta R$ rounds. Therefore, FedDST has an average upload and download overhead of $32(1-\gamma_{k,t})M_{\rm{ori}} +\frac{1}{\Delta R}M_{\rm{ori}} $.\\
\textbf{Computational overhead:} in all these three schemes, FLOPs for the $\ell$-th convolutional layer is $2J^2HWO_\ell O_{\ell-1}(1-\gamma_{k,t})$, where $O_{\ell-1}$ is the number of filters in the $(\ell-1)$-th layer, $J$ is the filter width, $H$ and $W$ are the height and width of the input feature maps; FLOPs for the $\ell$-th fully-connected (FC) layer is $2N_{\ell-1}N_\ell(1-\gamma_{k,t})$, where $N_{\ell-1}$ is the number of neurons in the $(\ell-1)$-th layer. The computational overhead is reduced to $(1-\gamma_{k,t})$ times the original amount, showing the same trend with $C_{k,t}$.
\end{remark}
}

\section{PROBLEM FORMULATION AND CONVERGENCE ANALYSIS}\label{sec:ConvrgAnalysis+ProblFrmlation}
In this section, we first show that the subnet's gradient vector can be approximated as a variance-bounded estimation of the original network's gradient vector. Based on this property, the convergence of FedDrop with an adaptive learning rate is characterized. Accordingly, a per-round learning loss reduction is obtained based on the theoretical analysis and an optimization problem is formulated to jointly design the adaptive dropout rate and bandwidth allocation for reducing the convergence latency.

\subsection{Variance-Bounded Gradient Estimation}
Let $\bm{\hat{g}}_k(\bm{\hat{w}}_{k}^{(t)})$  be the gradient vector of the subnet of device $k$ in an arbitrary round $t$.
\textcolor{blue}{Following \cite{wager2013dropout}, we adopt the Taylor expansion to approximate the gradients of a neural network. As such, the Taylor expansion of the local gradient vector with dropout at the reference point $\bm{w}^{(t)}$ is written as }
\begin{equation} \label{equ:tal}
    \begin{aligned}
     \bm{\hat{g}}_k(\bm{\hat{w}}_{k}^{(t)}) 
     & = 
     \bm{\tilde{g}}_k(\bm{w}^{(t)}) + O\left(\hat{\bm{w}}_k^{(t)}-\bm{w}^{(t)}\right)\\
     & +\bm{H}(\bm{w}^{(t)})\left(\hat{\bm{w}}_k^{(t)}-\bm{w}^{(t)}\right),
    \end{aligned}
\end{equation} 
where $ \bm{\tilde{g}}_k(\bm{w}^{(t)})$ is the stochastic gradient created by the local dataset of an arbitrary device, $\bm{H}(\bm{w}^{(t)})$ is the Hessian matrix of the loss function in terms of $\bm{w}^{(t)}$, and $O\left(\hat{\bm{w}}_k^{(t)}-\bm{w}^{(t)}\right)$ is the infinitesimal of higher order. 

Without loss of generality, the following assumptions are made for analyzing the property of $ \bm{\hat{g}}_k(\bm{\hat{w}}_{k}^{(t)})$, which are widely used in the existing FL literature \cite{castiglia2022compressed,jiang2023computation,liu2021adaptive}.
\begin{assumption}[Bounded Hessian] \label{assumption1} 
The F-norm of Hessian matrix $\bm{H}(\bm{w})$ is upper bounded, i.e.,
\begin{equation}
    \| \bm{H}(\bm{w}) \|_F^2 \leq A^2, \quad \forall \bm{w}.
\end{equation}
\end{assumption}

\begin{assumption}[Small Dropout Rate]
\label{assumption2}
The dropout rate of each device is small \footnote{The analysis is based on Assumption 2. However, the subsequent experimental results numerically show that even for the case with a relatively larger dropout rate, the convergence behavior is still observed.} so that
\begin{equation}
    \| \hat{\bm{w}}_k-\bm{w} \|^2 \rightarrow 0, \quad \forall \bm{w}.
\end{equation}
\end{assumption}

\begin{assumption}[Bounded weight]\label{assumption3}
The norm of the parameter vector is upper bounded, i.e.,
\begin{equation}
        \mathbb{E}\left[ \| \bm{w} \|^2 \right] \leq G^2, \quad \forall \bm{w}.
\end{equation}   
\end{assumption}
Assumption \ref{assumption2} requires that the dropout rates $\{\gamma_{k}\}$ should not be too large, otherwise, the federated training cannot converge. For example, if a dropout rate $\gamma_{k}$ approaches 1, almost all the neurons are dropped, leading to divergence without doubt. Under Assumption \ref{assumption2}, $O\left(\hat{\bm{w}}_k^{(t)}-\bm{w}^{(t)}\right)$ on the right side of \eqref{equ:tal} can be ignored, i.e.,
\begin{equation}
    \bm{\hat{g}}_k(\bm{\hat{w}}_{k}^{(t)})  \approx 
     \bm{\tilde{g}}_k(\bm{w}^{(t)})  +\bm{H}(\bm{w}^{(t)})\left(\hat{\bm{w}}_k^{(t)}-\bm{w}^{(t)}\right).
\end{equation}
Thereby, the following lemma is obtained.
\begin{lemma}\label{lem:expvar}
Under Assumptions \ref{assumption1}, \ref{assumption2}, and \ref{assumption3}, the gradient vector of a subnet is a variance-bounded estimation of the whole network's gradient vector:
\textcolor{blue}{
\begin{equation}\label{equ:expect}
    \mathbb{E}_{\bm{m}_{k}^{(t)}}\left[\hat{\bm{g}}_k (\hat{\bm{w}}^{(t)}_{k}) \right]=\Tilde{\bm{g}}_k (\bm{w}^{(t)}),
\end{equation}
\begin{equation}\label{equ:variance}
    \begin{aligned}
        \mathbb{D}_{\bm{m}_{k}^{(t)}}\left[\hat{\bm{g}}_k (\hat{\bm{w}}^{(t)}_{k}) \right]  & \leq \| \bm{H}(\bm{w}^{(t)}) \|_F^2 \cdot \| \bm{w}^{(t)} \|^2 \cdot \frac{\gamma_{k,t}}{1-\gamma_{k,t}}\\
        & \leq (AG)^2 \cdot \frac{\gamma_{k,t}}{1-\gamma_{k,t}}.
    \end{aligned}
\end{equation}
}
\end{lemma}
\begin{proof}
Consider an arbitrary gradient vector of the subnet. The expectation of \eqref{equ:tal} is obtained as
\begin{equation}\label{Eq:UnbiasedElements} 
    \begin{aligned}
    & \mathbb{E}_{\textcolor{blue}{\bm{m}_{k}^{(t)}}}\left[\hat{\bm{g}}_k (\hat{\bm{w}}^{(t)}_{k}) \right] \\
    & = \bm{\tilde{g}}_k(\bm{w}^{(t)}) +\bm{H}(\bm{w}^{(t)}) \mathbb{E}_{\textcolor{blue}{\bm{m}_{k}^{(t)}}}\left[\left(\hat{\bm{w}}_k^{(t)}-\bm{w}^{(t)}\right) \right]\\
    & = \bm{\tilde{g}}_k(\bm{w}^{(t)}) +\bm{H}(\bm{w}^{(t)}) \mathbb{E}_{\textcolor{blue}{\bm{m}_{k}^{(t)}}}\left[\bm{w}^{(t)} \circ \left(\bm{m}_k^{(t)}- \bm{1}\right) \right] \\
    & = \bm{\tilde{g}}_k(\bm{w}^{(t)}),
    \end{aligned}
\end{equation}
Based on \eqref{Eq:UnbiasedElements}, the gradient variance 
is upper-bounded as 
\begin{equation} \label{equ:var_bou}
    \begin{aligned}
        & \mathbb{E}_{\textcolor{blue}{\bm{m}_{k}^{(t)}}}\left[ \left( \bm{\hat{g}}_k(\bm{\hat{w}}_{k}^{(t)})-\bm{\tilde{g}}_k(\bm{w}^{(t)}) \right)^2\right]\\
	& \leq \| \bm{H}(\bm{w}^{(t)})     \|_F^2 \cdot \mathbb{E}_{\textcolor{blue}{\bm{m}_{k}^{(t)}}} \left[        \| \hat{\bm{w}}_k^{(t)}-         \bm{w}^{(t)}\| ^2 \right]\\
        &= \| \bm{H}(\bm{w}^{(t)}) \|_F^2 \cdot \| \bm{w}^{(t)} \|^2 \cdot \frac{\gamma_{k,t}}{1-\gamma_{k,t}}.
    \end{aligned}
\end{equation}
Then, integrating Assumptions \ref{assumption1} and \ref{assumption3} into \eqref{equ:var_bou}, the whole stochastic gradient vector of the subnet is written as \eqref{equ:variance}. This ends the proof.
\end{proof} 

\subsection{Convergence Analysis}

As DNNs are non-convex, the $l_2$-norm of the gradient vector is used for convergence analysis, similar to \cite{ghadimi2013stochastic}. Without loss of generality, the following three assumptions are made (see e.g., \cite{li2019convergence,allen20172,cao2021optimized}).

\begin{assumption}[Bounded Loss Function]\label{assumption4}
	The loss function $F(\bm{w})$ is lower bounded,
	i.e., 
    \begin{equation}\label{eqn:bounded}
		F(\bm{w}) \geq \mathcal{F}^*, \quad \forall \bm{w}.
	\end{equation}
\end{assumption}
\begin{assumption}[Smoothness]\label{assumption5}
	The loss function $F(\boldsymbol{z})$ is $ L $-smooth,
	i.e., for $\forall \bs{v}, \bs{w} \in \operatorname{dom}(f)$,
	\begin{equation}\label{eqn:smooth}
		F(\bs{v}) \leq F(\bs{w})+(\bs{v}-\bs{w})^{\top} \nabla F(\bs{w})+\frac{L}{2}\|\bs{v}-\bs{w}\|^{2}.
	\end{equation} 
\end{assumption}
\begin{assumption}[Variance Bound]\label{assumption6} 
    It is assumed that 
    $\{\tilde{\mathbf{g}}_k(\bs{w}) \}$ are independent and unbiased estimates of the ground-truth gradient with coordinate bounded variance, i.e.,
    \begin{equation}
        \left\{\begin{array}{l}
        \mathbb{E}\left[\bm{\tilde{g}}_k(\bs{w})\right]=\bm{g}(\bs{w}), \forall k, \\
        \mathbb{E}\left[\left\|\tilde{\bm{g}}_k(\bs{w})-\bm{g}(\bs{w})\right\|^2\right] \leq \frac{\sigma^2}{\left|\mathcal{D}_k\right|}, \forall k.
      \end{array}\right.
    \end{equation}
    where $\bm{g}(\bs{w})$  is the ground-truth gradient vector and $\sigma^2$ is the per-sample gradient variance.
\end{assumption}
Note that $\bs{w}$ in the above three assumptions represents the parameter vector of the whole network. 


\begin{lemma}\label{lem:adaptive}
Given learning rate $\eta= \frac{1}{3\sqrt{T}L}$. Based on Assumptions \ref{assumption4}--\ref{assumption6} and Lemma \ref{lem:expvar}, the loss function difference after an arbitrary round of FedDrop is 
\begin{equation}  \label{equ:lemma_opt}
    \begin{aligned}
        & \mathbb{E} \left[\mathcal{F}\left(\bm{w}^{(t+1)}\right)-\mathcal{F}\left(\bm{w}^{(t)}\right) \right] \\
        & \leq \left( -\frac{1}{3 \sqrt{T}L} +\frac{2}{9TL} \right) \|\bm{g} (\bm{w}^{(t)}) \|^2 + \frac{2K}{9TL} \frac{\sigma^2}{|\mathcal{D}|} \\
        &\quad  +\frac{A^2G^2}{9TL} \sum_{k=1}^K \frac{|\mathcal{D}_k|}{|\mathcal{D}|} \cdot \frac{\gamma_{k,t}}{1-\gamma_{k,t}},\\
    \end{aligned}
\end{equation}
where $T$ is the number of global rounds.
\end{lemma}
\begin{proof}
	Please refer to Appendix \ref{_proof_of_lemma_adaptive}.
\end{proof} 
It should be noticed that, to guarantee loss function reduction, namely, the loss function decreases in each round, the dropout rate should not be large; otherwise, the third term on the right side of \eqref{equ:lemma_opt} approaches infinity. This is also consistent with Assumption \ref{assumption2}. In this work, our objective is to minimize $\mathbb{E}\left[\mathcal{F}\left(\bm{w}^{(t+1)}\right)-\mathcal{F}\left(\bm{w}^{(t)}\right) \right]$ in each round $t$ via minimizing its upper bound in the right side of \eqref{equ:lemma_opt}.



\begin{theorem}\label{theorem:adaptive}
Based on Lemma \ref{lem:adaptive} and given $\eta= \frac{1}{3\sqrt{T}L}$, the ground-truth gradient vector via the training of FedDrop $\left\|\bm{g}\left(\bm{w}^{(t)}\right)\right\|^2 $ converges as follows:
	\begin{equation} \label{eqn:adaptive}
		\lim _{T \rightarrow+\infty} \frac{1}{T} \sum_{t=0}^{T-1}\left\|g\left(\bm{w}^{(t)}\right)\right\|^2
         \leq G_T,\\
	\end{equation}
	where  
\begin{equation}
    \begin{aligned}
        G_T 
        & = \lim _{T \rightarrow+\infty}\frac{1}{\sqrt{T}} \left\{ 9L\left[\mathcal{F}\left(\bm{w}^{(0)}\right) -\mathcal{F}^* \right] \right.\\
        & \left. \quad + 2K   \frac{\sigma^2}{|\mathcal{D}|} + \frac{A^2G^2}{T} \sum\limits_{t=0}^{T-1}\sum_{k=1}^K \frac{|\mathcal{D}_k|}{|\mathcal{D}|} \frac{\gamma_{k,t}}{1-\gamma_{k,t}}  \right\} = 0.\\
    \end{aligned}
\end{equation}
\end{theorem}
\begin{proof}
	Please refer to Appendix \ref{_proof_of_theorem:adaptive}.
\end{proof}
From Theorem \ref{theorem:adaptive}, it is observed that the convergence rate decreases with the increasing dropout rate of each device, since FedDrop reduces the communication and computational overhead in the per-round training, which inevitably incurs more training rounds for convergence. This thus provides the potential to reduce the overall convergence latency by making a trade-off between the per-round latency and the latency in the overall communication rounds by adjusting the dropout rates. Furthermore, the devices that cannot participate in FEEL due to limited computational or communication capabilities, can be involved in the training of FedDrop to contribute their data and computational resources.

\subsection{Problem Formulation}
In this work, we aim at minimizing the convergence latency, which is achieved by maximizing the learning loss reduction for each round under per-round latency constraint. For an arbitrary round $t$, the objective function for maximizing the learning loss reduction by joint design of dropout rate and bandwidth allocation is equivalent to 
\begin{equation}
    \min\limits_{\{ {\gamma}_{k,t},
				{\rho}_{k,t}\} }\;\;\mathbb{E} \left[\mathcal{F}\left(\bm{w}^{(t+1)}\right)-\mathcal{F}\left(\bm{w}^{(t)}\right) \right],
\end{equation}
where ${\gamma}_{k,t}$ and ${\rho}_{k,t}$ are the dropout rate and the bandwidth allocation ratio associated with device $k$. However, the exact value of the objective function above is intractable. Following a common approach in \cite{zeng2023federated,cao2021optimized,liu2023adaptive}, we minimize its upper bound approximation defined in the right side of \eqref{equ:lemma_opt} instead. Since the norm square of the ground-truth gradient vector $\|\bm{g} (\bm{w}^{(t)}) \|^2$, the gradient variance $\sigma^2$, and other parameters, i.e., $|\mathcal{D}|$, $|\mathcal{D}_k|$, $T$, and $L$ are constants, minimizing the right side of \eqref{equ:lemma_opt} is equivalent to 
\begin{equation}
     \min\limits_{\{{\gamma}_{k,t},
				{\rho}_{k,t}\}}\;\; \sum_{k=1}^K \frac{|\mathcal{D}_k|}{|\mathcal{D}|}\frac{\gamma_{k,t}}{1-\gamma_{k,t}}.
\end{equation}
As $\frac{\gamma_{k,t}}{1-\gamma_{k,t}}=-1+\frac{1}{1-\gamma_{k,t}}$, the minimization of the right side of \eqref{equ:lemma_opt} is further written as
\begin{equation}
     \min\limits_{\{{\gamma}_{k,t},
				{\rho}_{k,t}\}}\;\; \sum_{k=1}^K \frac{|\mathcal{D}_k|}{|\mathcal{D}|}\frac{1}{1-\gamma_{k,t}}.
\end{equation}
Due to the limited network resources, there are several constraints on the latency, energy consumption, bandwidth allocation, and dropout rate in each round $t$, as elaborated below.
\subsubsection{Per-Round Latency Constraint}
The latency of each device should not exceed the maximum permitted latency $T_0$ to complete this round. Based on the devices' latency $\{T_{k,t}\}$ derived in \eqref{Eq:DeviceLatency}, the latency constraint is given by 
\begin{equation} \label{Eq:con_lat}
 T_{k,t}^{\rm {com,dl}}+T_{k,t}^{\rm {cmp}}+T_{k,t}^{\rm {com,ul}} \leq T_0, \forall k.
\end{equation}
By substituting $T_{k,t}^{\rm {com,dl}}$ given in \eqref{eqn:communicationdl_time},
$T_{k,t}^{\rm {cmp}}$ given in \eqref{eqn:computation_time}, and $T_{k,t}^{\rm {com,ul}}$ given in \eqref{eqn:communicationul_time} into the latency constraint, we have
\begin{equation}
    \mathcal{C}_1: \; \frac{M_{k,t} Q}{\rho_{k,t}B}\left(\frac{1}{R_{k,t}^{\rm dl}}+\frac{1}{R_{k,t}^{\rm ul}}\right) + \frac{C_{k,t}|\mathcal{D}_k|}{f_{k,t}} \leq T_0, \forall k,
\end{equation}
\textcolor{blue}{where $M_{k,t} = (1-\gamma_{k,t})M_{\rm ori} $ is the number
of parameters of the subnet and $C_{k,t} = (1-\gamma_{k,t})C_{\rm {ori}} $ is the number of processor operations on one data sample}. 

\subsubsection{Energy Consumption Constraint}
The total energy consumption of each device
should be no larger than its energy budget $E_{k,0}$. Based on the energy consumption $\{E_{k,t}\}$ of devices given in \eqref{Eq:DeviceEnergy}, the energy consumption constraint is given by 
\begin{equation} \label{Eq:con_enr}
E_{k,t}^{\rm {com,ul}}+E_{k,t}^{\rm{cmp}}+\xi_k \leq E_{k,0}, \forall k.
\end{equation}
By substituting $E_{k,t}^{\rm{cmp}}$ in \eqref{eqn:computation_energy} and $E_{k,t}^{\rm{com,ul}}$ in  \eqref{eqn:communicationul_energy}, the energy consumption constraint is derived as
\begin{equation}
    \mathcal{C}_2:\;  \frac{M_{k,t} QP_{k,t}^{\rm{com,u l}}}{\rho_{k,t}BR_{k,t}^{\rm{ul}}} + C_{k,t}\Omega_k|\mathcal{D}_k|f_{k,t}^2 +\xi_k \leq E_{k,0}, \forall k.
\end{equation}

\subsubsection{Bandwidth Allocation Constraint}
The total allocated bandwidth to all devices should be no larger than the system bandwidth:
\begin{equation}
\mathcal{C}_3:\; \left\{\begin{array}{l}
\sum_{k=1}^K  \rho_{k, t}\leq 1, \\
\rho_{k, t} \geq 0, \; \forall k.
\end{array}\right.
\end{equation}

\subsubsection{Dropout Rate Constraint}
Based on the definition, the dropout rate of each device in each communication round should be limited between 0 and $\theta$, where $\theta$ is the maximum dropout rate ensuring the loss function reduction, namely,
\begin{equation}
\mathcal{C}_4:\; 0 \leq \gamma_{k,t} < \theta, \forall k.
\end{equation}

Under the constraints above, the optimization problem is formulated as follows:
\begin{equation}\label{equ:P_1}
\mathscr P_1 \quad
	\begin{array}{cl}
		\underset
		{
			\underset
			{
				k \in \mathcal{K} 
			}
			{
				\left \{
				{\gamma}_{k,t},
				{\rho}_{k,t}
				\right \}
			}
		}
		{
			\min
		}
		&	
		\sum_{k=1}^K \frac{|\mathcal{D}_k|}{|\mathcal{D}|}\frac{1}{1-\gamma_{k,t}},
		\\
		\text{subject to}
		& \mathcal{C}_1\sim \mathcal{C}_4.
	\end{array}
\end{equation}
\textcolor{blue}{We illustrate the trade-off between communication and computation through $\mathcal{C}_1$ and $\mathcal{C}_2$, where both time and energy constraints are considered. We also use dropout as a controller to collaborate with resource allocation strategies.} The convergence latency is reduced by solving $\mathscr P_1$  in each round. In the sequel, an arbitrary round is considered and the notation $t$ is omitted for simplicity. 

\section{Joint Optimization of Dropout Rate and Resource Allocation}
\label{sec:OptScheme}

\textcolor{blue}{$\mathscr P_1$ is not a convex problem due to the concaveness of $\mathcal{C}_1$ and $\mathcal{C}_2$. To tackle it, the following variable transformation is first made, 
$    x_k = \sqrt{1 - \gamma_k},$
where $x_k$ shows the probability of retaining the weights. It follows that $\mathcal{P}_1$ is equivalently derived as
\begin{equation}\label{equ:P_2}
\mathscr P_2 \quad
	\begin{array}{cl}
		\underset
		{
			\underset
			{
				k \in \mathcal{K} 
			}
			{
				\left \{
				{\gamma}_{k,t},
				{\rho}_{k,t}
				\right \}
			}
		}
		{
			\min
		}
		&	
		\sum_{k=1}^K \frac{|\mathcal{D}_k|}{|\mathcal{D}|}\frac{1}{x_k^2},
		\\
		\text{subject to}
		& \mathcal{C}_1\sim \mathcal{C}_4.
	\end{array}
\end{equation}
} Then, a low-complexity algorithm is proposed to obtain the optimal dropout rates and bandwidth allocation ratios in closed forms. 



\begin{lemma} \label{lem:cvx}
$\mathscr P_2$ is a convex optimization problem. 
\end{lemma}
\begin{proof}
Obviously, the objective function is convex. $\mathcal{C}_3$ and $\mathcal{C}_4$ form convex sets.  $\mathcal{C}_1$ and $\mathcal{C}_2$ form convex sets if the following function in terms of $\gamma_k$ and $\rho_k$ is convex:
\begin{equation}\label{Eq:ConvexProof}
 \dfrac{M_k}{\rho_k} = \dfrac{x_k^2M_{\rm ori}  }{ \rho_k }, 
\end{equation}
where $x_k \in (0,1]$ and  $\rho_k \in (0,1]$. Its Hessian matrix is 
$$
    \left[
        \begin{array}{clr}
        \dfrac{2 M_{\rm ori} }{\rho_k} & -\dfrac{2 M_{\rm ori} x_k }{\rho_k^2}\\
        -\dfrac{2 M_{\rm ori} x_k }{\rho_k^2}  & \dfrac{2 M_{\rm ori} x_k^2}{\rho_k^3}
        \end{array}
    \right]
$$
which is non-negative. Hence, $\mathscr P_1$ is a convex optimization problem, thus completing the proof.
\end{proof}
Then, the primal-dual method is used to obtain its optimal solution. Specifically, the Larange function for $\mathscr P_1$ is given as
\begin{equation} \label{eqn:KKT}
    \begin{aligned}
        &\mathcal{L}
         =\sum_{k=1}^K \frac{|\mathcal{D}_k|}{|\mathcal{D}| } \frac{1}{x_k^2} + \mu\left(\sum_{k=1}^K \rho_{k}-1\right)\\
        & +\sum_{k=1}^K \lambda_{k}\left(\frac{M_{k} Q}{r_{k}^{\rm{dl}}}+\frac{M_{k} Q}{r_{k}^{\rm{ul}}} + \frac{C_{k}|\mathcal{D}_k|}{f_{k}}-T_0\right)\\
        & +\sum_{k=1}^K \nu_{k}\left(\frac{P_{k}^{\rm{com,ul}}M_{k} Q}{r_{k}^{\rm{ul}}} +\Omega_kC_{k}|\mathcal{D}_k|f_k^2 +\xi_k-E_{k,0} \right),\\
    \end{aligned}
\end{equation}
where $\mu$, $\{\lambda_{k}\}$ and $\{\nu_{k}\}$ are the Lagrangian multipliers. The Karush-Kuhn-Tucker (KKT) conditions conclude $\frac{\partial \mathcal{L}}{\partial x_{k}}=0$ and $\frac{\partial \mathcal{L}}{\partial \rho_{k}}=0$.


\begin{figure*}[ht] 


\centering
\begin{equation} 
\label{equ:gamma1}
    \begin{aligned}
        x_k=\sqrt[4]{ \frac{2|\mathcal{D}_k|}{ |\mathcal{D}| \left(\lambda_k\left(\frac{2\textcolor{blue}{M_{\text{ori}}}Q}{\rho_{k}BR_{k}^{\rm{dl}}}+\frac{2\textcolor{blue}{M_{\text{ori}}}Q}{\rho_{k}BR_{k}^{\rm{ul}}}+\frac{2\textcolor{blue}{C_{\text{ori}}}|\mathcal{D}_k|}{f_{k}}\right) +\nu_k \left(\frac{2\textcolor{blue}{M_{\text{ori}}}QP_{k}^{\rm{com,ul}}}{\rho_{k}BR_{k}^{\rm{ul}}}+2\Omega_k\textcolor{blue}{C_{\text{ori}}}|\mathcal{D}_k|f_{k}^2\right)\right) } }.
    \end{aligned}
\end{equation}

\centering
\begin{equation}
\label{equ:rho1}
    \begin{aligned}
        \rho_k=\sqrt{\frac{\left(\lambda_k \left( \frac{M_kQ}{BR_k^{ul}} + \frac{M_kQ}{BR_k^{dl}} \right) + \nu_k \frac{P_{k}^{\rm{com,ul}}M_{k}Q}{BR_{k}^{\rm{ul}}} \right) }{\mu} }.
    \end{aligned}
\end{equation}
\vspace*{8pt}
\hrulefill

\end{figure*}
\textcolor{blue}{We obtain the closed-form solution for the optimal $\gamma_k$ and $\rho_k$ as given in \eqref{equ:gamma1} and \eqref{equ:rho1}.} It is observed that when the device has better uplink and downlink channel conditions, it should be allocated with a smaller dropout rate and bandwidth. Besides, the relationship between the optimal dropout rate for each device and its computation speed is non-monotonic as a higher computation speed leads to lower latency but higher energy consumption. Moreover, the optimal bandwidth allocation ratio for each device \textcolor{blue}{shows the opposite trend} to its dropout rate. A lower dropout rate leads to higher loads for communication and computation. 
Based on the closed-form solution in \eqref{equ:gamma1} and \eqref{equ:rho1}, a primal-dual method based low-complexity algorithm is proposed, as detailed in Algorithm \ref{Alg:Solution}, where  $\eta_{\mu}$, $\eta_{\lambda k}$, and $\eta_{\nu k}$ are the step sizes for updating $\mu$, $\lambda_k$, and $\nu_k$,  and 
\begin{equation}
\left\{
\begin{aligned}
&  \delta_{\lambda,k,i} =  \frac{M_{k} Q}{r_{k}^{\rm{dl}}}+\frac{M_{k} Q}{r_{k}^{\rm{ul}}} + \frac{C_{k}|\mathcal{D}_k|}{f_{k}}-T_0,\\
& \delta_{\nu,k,i} = \frac{P_{k}^{\rm{com,ul}}M_{k} Q}{r_{k}^{\rm{ul}}} +\Omega_kC_{k}|\mathcal{D}_k|f_k^2 +\xi_k-E_{k,0}
\end{aligned}
\right.
\end{equation}
are the partial derivatives of $\mathcal{L}$ in terms of $\lambda_k^{(i)}$ and $\nu_k^{(i)}$, respectively. Based on the closed-form expressions for $\gamma_{k}$ and $\rho_k$, the complexity of Algorithm \ref{Alg:Solution} is only $ \mathcal{O}(K^2)$, where $K$ is the number of devices.
\begin{algorithm}[]
	\caption{Joint Design of Dropout Rate and Bandwidth Allocation}\label{Alg:Solution}
	\LinesNumbered
	\KwIn{ \{$R_k^{\rm{dl}}$\},  \{$R_k^{\rm{ul}}$\}  \{$P_k^{\rm{com,ul}}$\}, \{$f_k$\}, $T_0$, \{$E_{k,0}$\}.} 

     \textbf{Initialize} \{$\mu^{(0)}$\},  \{$\lambda_k^{(0)}$\}, \{$\nu_k^{(0)}$\}, the step sizes \{$\eta_{\mu}$\}, \{$\eta_{\lambda k}$\}, \{$\eta_{\nu k}$\}, and $i=0$. \\
    \textbf{Loop}.\\  
    Calculate \{$\gamma_{k}$\} and \{$ \rho_{k}$\} using \eqref{equ:gamma1} and \eqref{equ:rho1}.\\  
    Update the multipliers as
    $$
    \mu^{(i+1)}= \max \left\{ \mu^{(i)} + \eta_{\mu} (\sum_{k=1}^K \rho_{k}-1),0 \right\}, 
    $$
    $$
    \lambda_k^{(i+1)}= \max \left\{ \lambda_k^{(i)} +  \eta_{\lambda k} \delta_{\lambda,k,i},\; 0 \right\}, \forall k,
    $$    
    \[
    	\begin{split}
	\nu_k^{(i+1)} 
       = ~& \max \left\{ \nu_k^{(i)} + \eta_{\nu k}\delta_{\nu,k,i}, \; 0 \right\}, \forall k,
	    \end{split}
    \]   
    \\
    $i = i + 1$.\\
    \textbf{Until Convergence}.\\
    \KwOut{\{$\gamma_{k}$\} and \{$ \rho_{k}$\}.}
\end{algorithm}


\section{Simulation Results}

\subsection{Simulation Settings}
Consider a single-cell FedDrop system with one server located at the center and $K = 10$ devices. 
\textcolor{blue}{The uplink and downlink channel gains of devices are assumed to follow Rician fading with a $\kappa$-factor of 10 and an average path loss of $10^{-3}$. Rician fading channel model is considered which includes a dominant line-of-sight (LoS) component and scattered multipath components.} Other related parameters are listed in Table \ref{table:para} by default. The simulation is conducted in two scenarios. \textcolor{blue}{One is the underfitting scenario with a simple model, LeNet, while the other is the overfitting scenario with a complex model, AlexNet.
Both of these two models are trained on CIFAR100. The details of two models are elaborated below.}


\begin{itemize}
    \item \textcolor{blue}{LeNet: A convolutional neural network with two convolutional layers and two FC layers is adopted. The kernel sizes are both $5 \times 5$. Max pooling operations are conducted following each convolutional layer. The activation function is Tanh. The sizes of these two FC layers are $120$ and $84$, respectively. }

    \item \textcolor{blue}{AlexNet: A convolutional neural network with five convolutional layers and two FC layers is adopted. The kernel sizes are both $3 \times 3$. Max pooling operations are conducted following the first, second, and fifth convolutional layers. The activation function is ReLU. The sizes of these two FC layers are both $4096$.}
\end{itemize}


\begin{table}[htbp]
	\centering
	\caption{Simulation parameter settings}
	\label{table:para}
	\begin{tabular}{|c|c|c|}
		\hline
		Parameter          & Description                                                              & Value                                                                            \\ \hline
		$f_k$          & \begin{tabular}[c]{@{}c@{}}  CPU frequency \\of device $k$ ($\mathrm{Hz}$)       \end{tabular}                               & $\operatorname{Unif}(0.7,1) \times 7 \times 10^9 $                                                            \\ \hline
		$p_k$          & \begin{tabular}[c]{@{}c@{}} transmission \\power of device $ k $ ($\mathrm{W}$)    \end{tabular}                                & $\operatorname{Unif}(0.3,1) \times 10^{-2}$                                                            \\ \hline
		$\Omega_k$            & \begin{tabular}[c]{@{}c@{}}CPU constant\\ of device $ k $   \end{tabular}                                & $\operatorname{Unif}(0.3,1) \times 10^{-26}  $                                                                                      \\ \hline
		$\sigma^2$            & \begin{tabular}[c]{@{}c@{}} Noise spectral \\density  ($\mathrm{W/Hz}$)  \end{tabular}                                                & $10^{-13}  $                                                                   \\ \hline
		$Q$              & \begin{tabular}[c]{@{}c@{}} Transmission \\bits per parameter \end{tabular}                              & $256   $                                                                           \\ \hline
	\end{tabular}
	
\end{table}
\subsection{Effects of Dropout Rate}

\begin{table*}[htbp]
\centering
\caption{\textcolor{blue}{Accuracy of FedDrop and other methods given the same sparsity}}
\label{table:baseline}
\setlength{\tabcolsep}{4mm}
\begin{tabular}{|c|c|c|c|c|c|c|c|c|c|c|c|}
\hline
Model                    & Data                     & Method  & 0     & 0.05           & 0.1            & 0.15           & 0.2            & 0.25           & 0.3            & 0.35           & 0.4            \\ \hline
\multirow{3}{*}{LeNet}   & IID                      & FedDrop & 26.17 & 25.70          & 24.83          & 24.69          & 24.33          & 22.32          & 20.94          & 20.52          & 19.80          \\ \cline{2-12} 
                         & \multirow{2}{*}{Non-IID} & FedDST  & /     & 10.70          & 14.32          & 13.09          & 11.66          & 13.85          & 13.39          & 11.60          & 14.20          \\ \cline{3-12} 
                         &                          & FedDrop & 25.19 & \textbf{25.10} & \textbf{24.03} & \textbf{23.40} & \textbf{22.57} & \textbf{21.76} & \textbf{20.62} & \textbf{19.29} & \textbf{19.09} \\ \hline
\multirow{3}{*}{AlexNet} & IID                      & FedDrop & 26.03 & 30.87          & 32.59          & 33.42          & 32.88          & 32.86          & 31.62          & 30.97          & 30.41          \\ \cline{2-12} 
                         & \multirow{2}{*}{Non-IID} & FedDST  & /     & \textbf{32.60} & 29.74          & 32.68          & 31.69          & 30.72          & \textbf{32.24} & \textbf{31.10} & \textbf{31.72} \\ \cline{3-12} 
                         &                          & FedDrop & 25.91 & 29.11          & \textbf{32.16} & \textbf{32.77} & \textbf{32.15} & \textbf{31.88} & 31.50          & 30.66          & 30.20          \\ \hline
\end{tabular}
\end{table*}

In this part, all devices perform a uniform dropout rate to show the effects of dropout rate on the convergence rate. Specifically, different subnets are generated using the same dropout rate and trained in each round. \textcolor{blue}{We use a $Dirichlet (0.1)$ distribution for each class to distribute a non-IID fashion, as in \cite{bibikar2022federated}. The data distribution is shown in Fig. \ref{dirichlet distribution}, where a larger circle means more data samples in the class.} 

\begin{figure*}[htbp]
\centering
\includegraphics[scale=0.4]{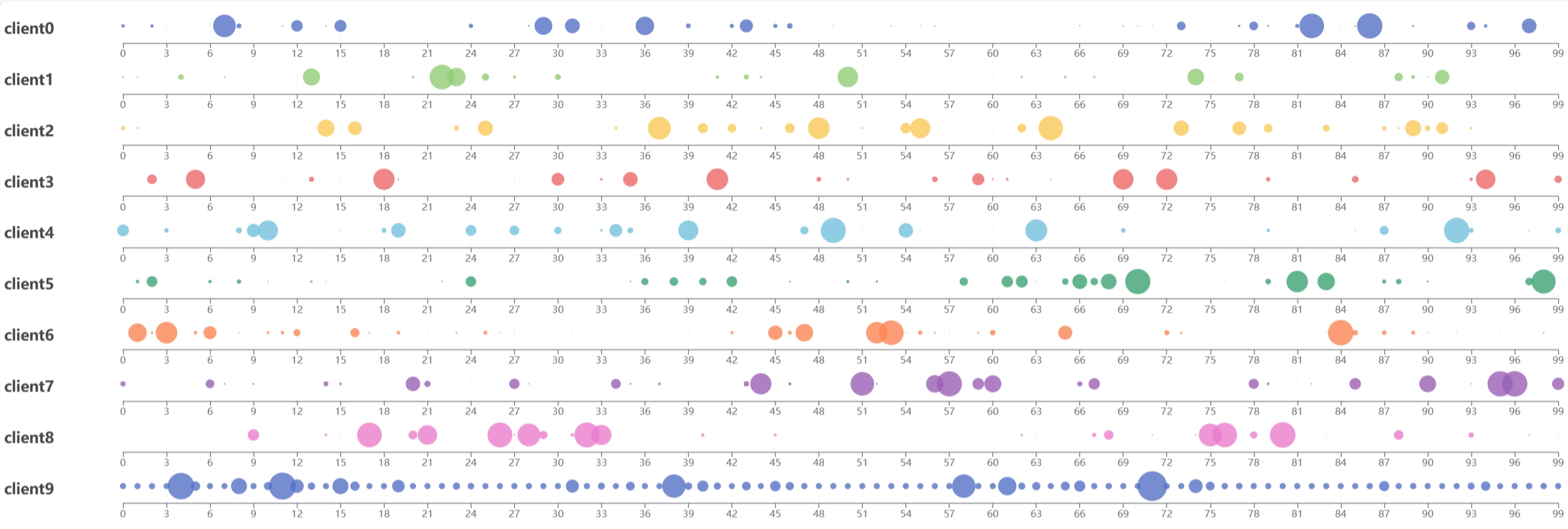}
\caption{\textcolor{blue}{The Dirichlet distribution of data on all the clients.}}
\label{dirichlet distribution}
\end{figure*}

\textcolor{blue}
{
\subsubsection{Comparison of Dropout Rate with other baselines} We compare FedDrop to competitive state-of-the-art baselines. FedDST \cite{bibikar2022federated} is included, while algorithms producing separate models for each client are not considered. As shown in Table \ref{table:baseline}, FedDrop consistently outperforms LeNet, and provides better performance with lower sparsity in AlexNet, under non-IID data distribution. This is because, with a rather complex network and a higher sparsity, knowing more knowledge in advance can help improve the regularization effect. While this procedure will incur more overheads and complexity.
}

\subsubsection{Effects of Dropout Rate on Performance}
\begin{itemize}
    \item \textbf{Underfitting Scenario:} The effects of different dropout rates on testing accuracy value in underfitting scenarios are shown in the first four lines. It is observed that the testing accuracy decreases with an increasing dropout rate under both data distributions. \textcolor{blue}{This is because a simple model is trained with sufficient data samples, resulting in underfitting.} Dropout in the underfitting scenario degrades the presentation ability of the AI model. 
    
    \item \textbf{Overfitting Scenario:} The effects of different dropout rates on testing accuracy value in overfitting scenarios are shown in the last three lines. \textcolor{blue}{In this scenario, the model without dropout is overfitted, i.e., a complex model is trained on relatively insufficient data samples. As a result, the testing accuracy of the training with dropout rates of 0.15 is better, since the overfitting is alleviated in the former case \cite{srivastava2014dropout}. However, as the dropout rate gets larger (i.e. $0.4$), the testing accuracy decreases since the presentation ability of the network decreases. } 
\end{itemize}

\subsection{Effects of Network Resources on FedDrop Learning Performance}
In this part, the effects of per-round latency and system bandwidth on the learning performance are investigated. Two baseline algorithms are adopted for comparison with our proposed scheme, which are described below.
\begin{itemize}
    \item \textbf{Proposed Scheme}: The optimal scheme proposed in Algorithm \ref{Alg:Solution}. 
    \item \textbf{Bandwidth-aware Scheme}: The bandwidth is randomly allocated and the dropout rates of all devices are optimized.
    \item \textbf{Scheme without Dropout:} Sufficient bandwidth is allocated and dropout is not considered, which is an ideal benchmark used to compare with the proposed scheme. 
\end{itemize}

\subsubsection{Effects of Per-round Latency}
Fig. \ref{Fig.cifart_len} and Fig. \ref{Fig.cifart_alex} present the effects of per-round latency in underfitting and overfitting scenarios, respectively. \textcolor{blue}{The proposed scheme consistently outperforms the bandwidth-aware scheme, achieving convergence in fewer rounds. Notably, as per-round latency increases, the performance gap narrows between the proposed scheme and the scheme without dropout. In both underfitting and overfitting scenarios, the convergence round decreases with an increasing per-round latency. This is because, with a larger per-round latency, the dropout rate decreases, effectively reducing the amount of noise added to the training process and consequently accelerating the convergence rate. In the overfitting scenario, the convergence slope exhibits two distinct phases: the transition from initial underfitting to optimal performance, followed by a shift from optimal to overfitting. }



\subsubsection{Effects of System Bandwidth}
Fig. \ref{Fig.cifarb_len} and Fig. \ref{Fig.cifarb_alex} present the effects of system bandwidth in underfitting and overfitting scenarios, respectively. \textcolor{blue}{It is observed that the proposed scheme always outperforms the bandwidth-aware scheme, with fewer rounds to converge. Furthermore, with the increase of system bandwidth, the gap decreases between the proposed scheme and the scheme without dropout. In both underfitting and overfitting scenarios, the convergence round decreases with a system bandwidth. The reasons for the above phenomenon are the same as that of the effects of per-round latency. } 

\begin{figure*}[htbp]
\begin{center}
\subfigure[\textcolor{blue}{Convergence round v.s. per-round latency in underfitting scenarios}]
{
\label{Fig.cifart_len}
\includegraphics[width=0.41\textwidth]{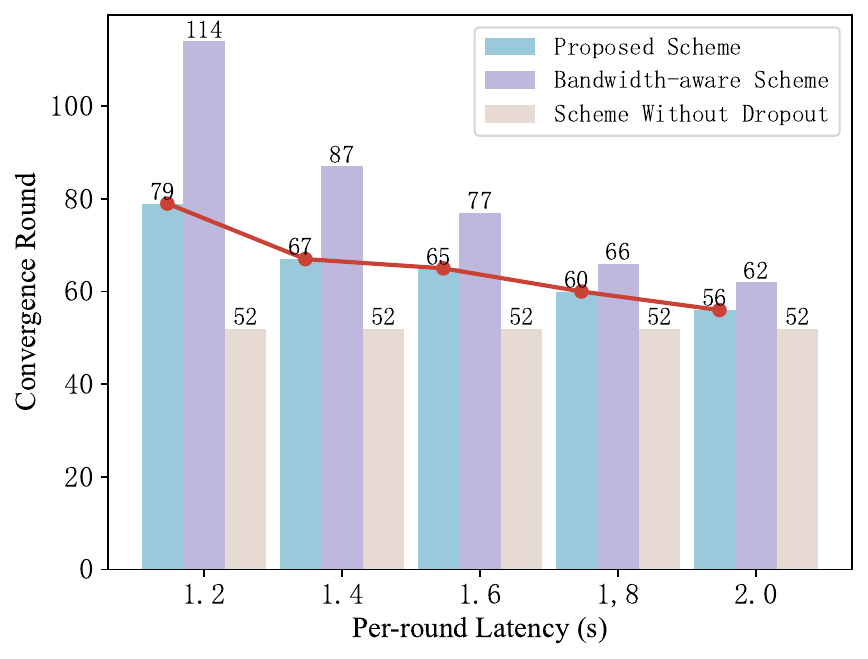}
}
\subfigure[\textcolor{blue}{Convergence round v.s. per-round latency in overfitting scenarios}]
{
\label{Fig.cifart_alex}
\includegraphics[width=0.41\textwidth]{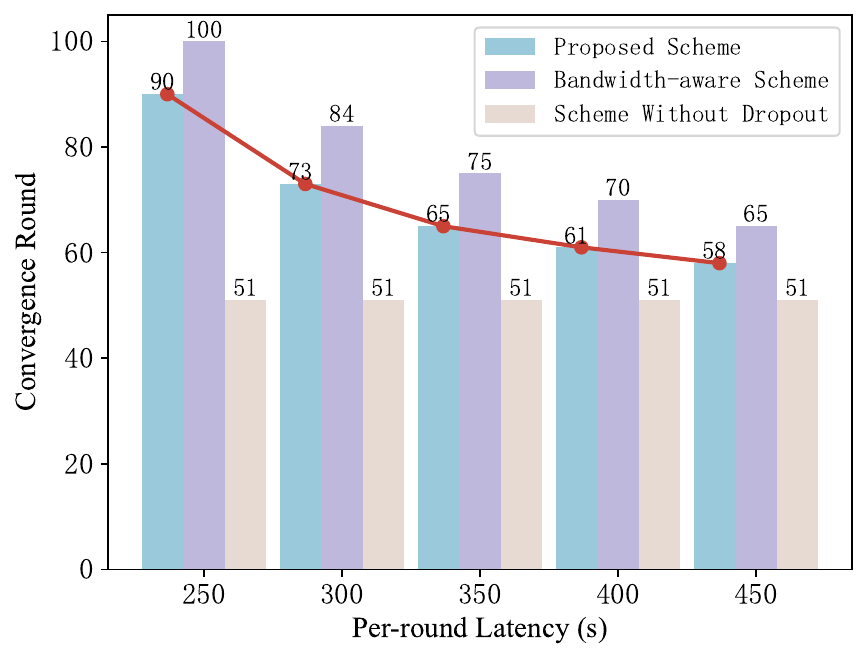}
}
\caption{\textcolor{blue}{Effects of per-round latency on convergence round in underfitting and overfitting scenarios, respectively.}}
\label{Fig.cifart}
\end{center}
\end{figure*}

\begin{figure*}[htbp]
\begin{center}
\subfigure[\textcolor{blue}{Convergence round v.s. system bandwidth in underfitting scenarios}]
{
\label{Fig.cifarb_len}
\includegraphics[width=0.41\textwidth]{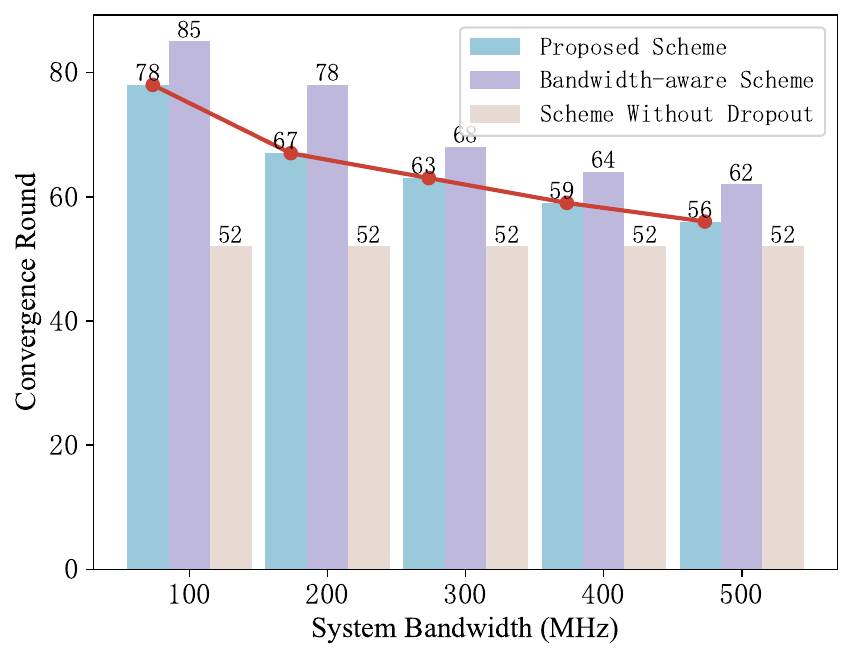}
}
\subfigure[ \textcolor{blue}{Convergence round v.s. system bandwidth in overfitting scenarios}]
{
\label{Fig.cifarb_alex}
\includegraphics[width=0.41\textwidth]{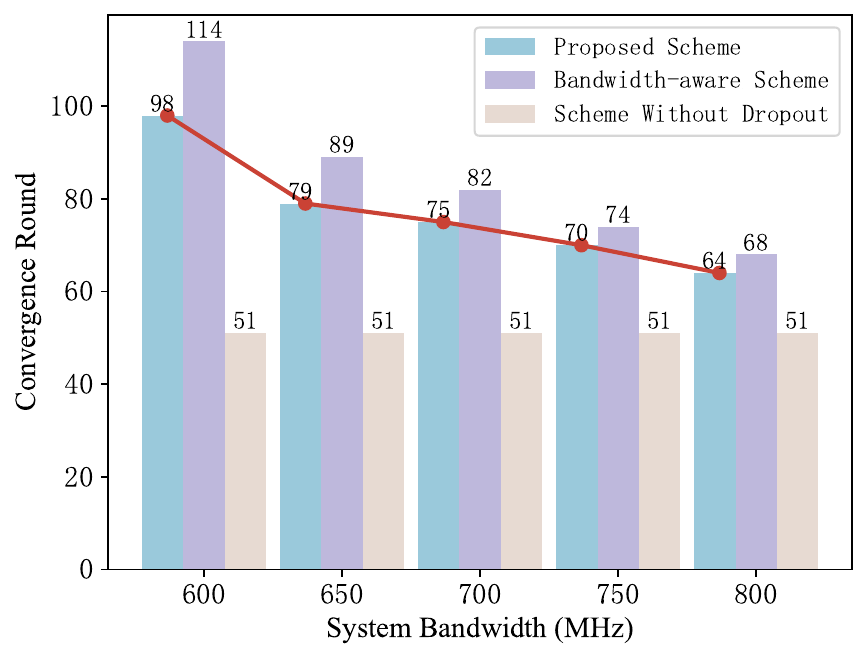}
}
\caption{ \textcolor{blue}{Effects of system bandwidth on convergence round in underfitting and overfitting scenarios, respectively.}}
\label{Fig.cifarb}
\end{center}
\end{figure*}

\subsubsection{Testing Accuracy}
\textcolor{blue}{The effects of per-round latency on testing accuracy are shown in Fig. \ref{Fig.cifar_acc}.} It is obvious that the proposed scheme always outperforms the bandwidth-aware scheme. The performance gain remains increasing when the latency continuously increases. This is because, with a looser constraint, a smaller dropout rate can meet the transmission requirements, thus leaving more parameters in the network. As shown in Fig. \ref{Fig.cifarb_acc}, in the overfitting scenario, the proposed scheme outperforms the scheme without dropout, as overfitting is effectively mitigated. \textcolor{blue}{Specifically, when the constraint is large enough, the performance gain of the proposed scheme and bandwidth-aware scheme is almost equal, because the dropout rates in both schemes are almost $0$, regardless of the bandwidth optimization.}

\begin{figure*}[htbp]
\begin{center}
\subfigure[\textcolor{blue}{Testing accuracy v.s. per-round latency in underfitting scenarios}]
{
\label{Fig.cifart_acc}
\includegraphics[width=0.41\textwidth]{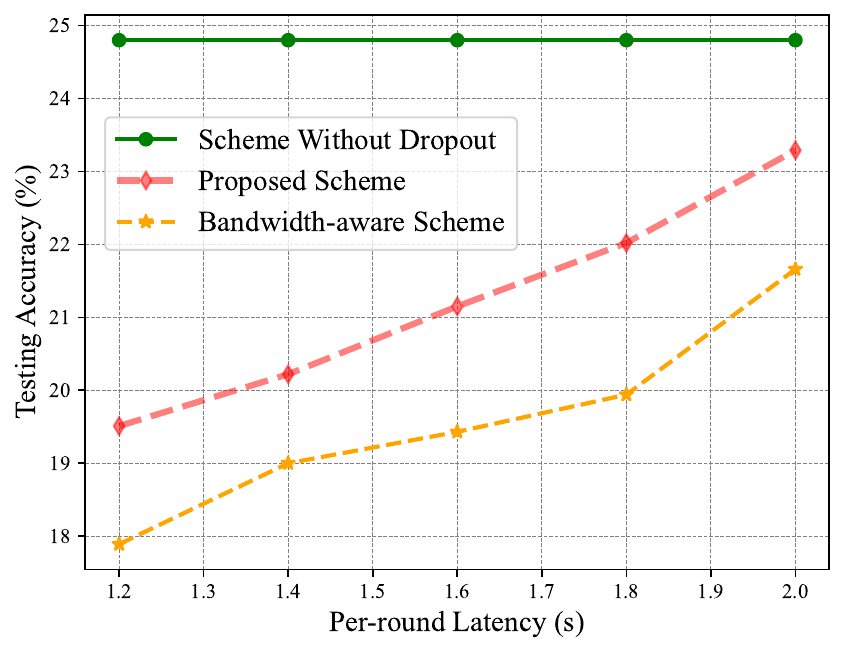}
}
\subfigure[\textcolor{blue}{Testing accuracy v.s. per-round latency in overfitting scenarios}]
{
\label{Fig.cifarb_acc}
\includegraphics[width=0.41\textwidth]{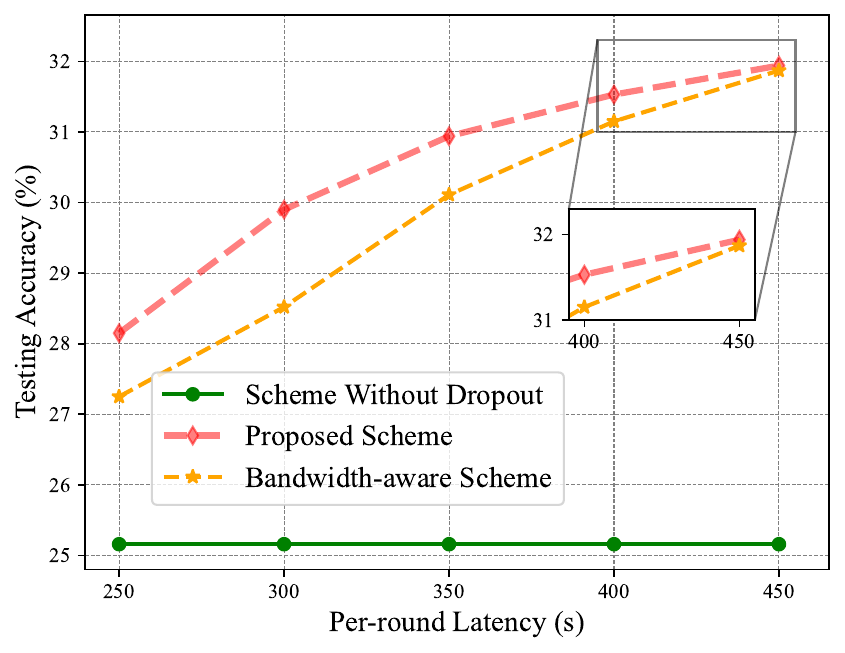}
}
\caption{\textcolor{blue}{Effects of per-round latency on testing accuracy in underfitting and overfitting scenarios, respectively.}}
\label{Fig.cifar_acc}
\end{center}
\end{figure*}

\subsection{Effects of Network Resources on FedDrop Performance}
\textcolor{blue}{We also change the edge learning environments, with different device numbers, computing frequency, and transmission power across different rounds to better reflect the applicability of the algorithm in realistic complex environments. The device number varies between 5 and 10, with the ranges of computing frequency and transmission power for each device detailed in Table \ref{table:para}. As shown in Fig. \ref{Fig.iid}, the testing accuracy of IID data is better than that of the non-IID data. This superiority is fundamentally rooted in the skewed data distribution characteristic of non-IID scenarios, where certain classes are underrepresented or completely absent at various data points. Moreover, algorithms exhibit slower convergence rates when processing non-IID data, necessitating additional training rounds to achieve comparable performance levels. Furthermore, the proliferation of network resources exacerbates the performance disparity between IID and non-IID data, as the intrinsic skewness of non-IID data makes it increasingly challenging for models to attain global optimality in resource-rich network environments.}

\begin{figure*}[htbp]
\begin{center}
\subfigure[\textcolor{blue}{Testing accuracy v.s. per-round latency in underfitting scenarios}]
{
\label{Fig.lenetiid}
\includegraphics[width=0.41\textwidth]{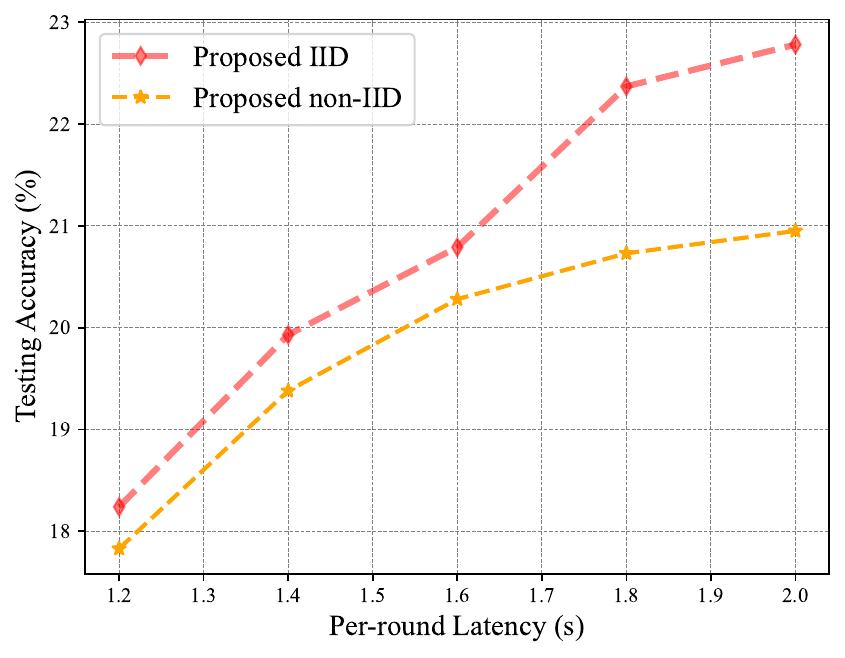}
}
\subfigure[\textcolor{blue}{Testing accuracy v.s. per-round latency in overfitting scenarios}]
{
\label{Fig.alexiid}
\includegraphics[width=0.41\textwidth]{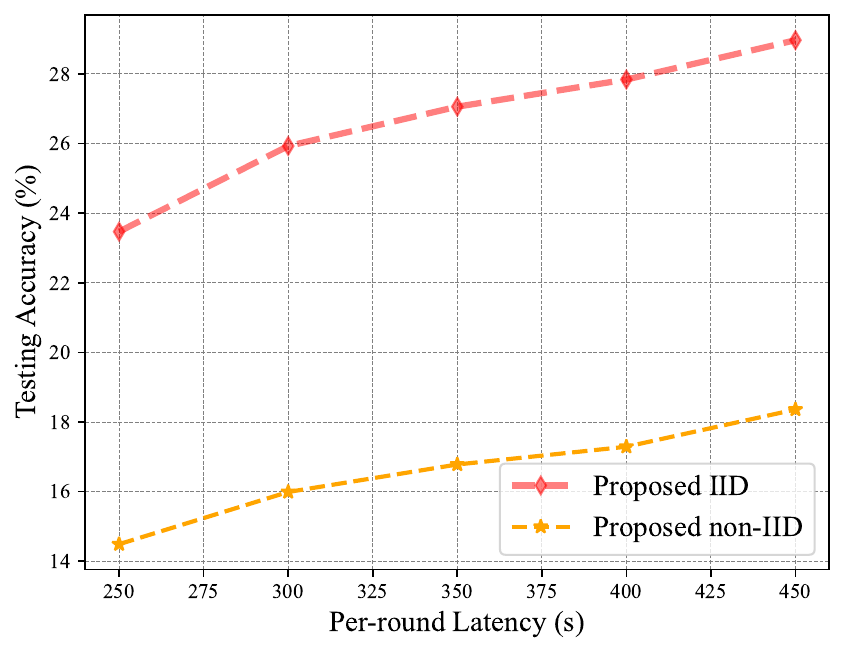}
}
\caption{\textcolor{blue}{Effects of dynamic wireless environments in underfitting and overfitting scenarios, respectively.}}
\label{Fig.iid}
\end{center}
\end{figure*}

\section{Conclusion}
In this paper, we propose a detailed theoretical analysis for the convergence performance of the FedDrop. The convergence upper bound is further utilized to formulate an optimization problem. Under the per-round latency and overall power constraints, we obtain the optimal dropout rate and bandwidth resource allocation in closed forms. Extensive simulation results show that the proposed scheme achieves a higher accuracy and lower computational and communication overhead than the bandwidth-aware scheme and the scheme without dropout. In this paper, we only consider the DNNs, while the case of the LLMs is left for future research. Moreover, other promising techniques such as compression and Air-comp can also be added to the FedDrop scheme to further improve the learning performance.

\appendices

\section{PROOF OF LEMMA \ref{lem:adaptive}}\label{_proof_of_lemma_adaptive}

We now analyze the convergence of FedDrop with respect to the dropout rate $\gamma$. Throughout the proof, we use the following inequalities frequently.

From Jensen's inequality, for any $\bm{g}_k \in \mathbb{R}^d$, $k \in \{1,2,\ldots,K\}$, we have
$    \left\|\frac{1}{K} \sum_{k=1}^K \bm{g}_k \right\|^2 \leq \frac{1}{K} \sum_{k=1}^K \| \bm{g}_k \|^2,$
and directly gives
$    \left\| \sum_{k=1}^K \frac{|\mathcal{D}_k|}{|\mathcal{D}|} \bm{g}_k \right\|^2 \leq \sum_{k=1}^K  \frac{|\mathcal{D}_k|}{|\mathcal{D}|} \| \bm{g}_k \|^2.$
Peter-Paul inequality (also known as Young's inequality) is shown as
$    <\bm{g}_1,\bm{g}_2> \leq \frac{1}{2}\| \bm{g}_1 \|^2 + \frac{1}{2}\| \bm{g}_2 \|^2,$
and for any constant $s>0$ and $\bm{g}_1, \bm{g}_2 \in \mathbb{R}^d$, we have
$    \| \bm{g}_1 + \bm{g}_2 \|^2 \leq (1+s) \| \bm{g}_1 \|^2 + (1+\frac{1}{s}) \| \bm{g}_2 \|^2.$

Consider an arbitrary round, say the $t$-th. The gradient descent for the FedDrop scheme is
\begin{equation} \label{equ:basic1}
    \bm{w}^{(t+1)}=\bm{w}^{(t)}-\eta \sum_{k=1}^K \frac{|\mathcal{D}_k|}{|\mathcal{D}|} \hat{\bm{g}}_k\left(\bm{\hat{w}}_{k}^{(t)}\right),
\end{equation}
where $\hat{\bm{g}}_k\left(\bm{\hat{w}}_{k}^{(t)}\right)$ is the gradient vector of the subnet in device $k$, and $\eta$ is the learning rate. Then, according to Assumption \ref{assumption5}, it is derived that
\begin{equation} \label{equ:Exp}
    \begin{aligned}
        & \mathbb{E} \left[\mathcal{F}\left(\bm{w}^{(t+1)}\right)-\mathcal{F}\left(\bm{w}^{(t)}\right) \right] \\
          \leq  &  -\eta\bm{g}^{\top}(\bm{w}^{(t)})\sum_{k=1}^K\frac{|\mathcal{D}_k|}{|\mathcal{D}|}\mathbb{E}\left[\hat{\bm{g}}_k (\bm{\hat{w}}_{k}^{(t)}) \right]   \\
          & + \frac{L\eta^2}{2}\mathbb{E}\left[ \left\| \sum_{k=1}^K\frac{|\mathcal{D}_k|}{|\mathcal{D}| }\hat{\bm{g}}_k (\bm{\hat{w}}_{k}^{(t)}) \right\|^2\right].\\
    \end{aligned}
\end{equation}
For the first term on the right side of \eqref{equ:Exp}, substitute \eqref{equ:expect} into it, we obtain 
\begin{equation} \label{equ:expect_1}
    \begin{aligned}
        & -\eta \bm{g}^T(\bm{w}^{(t)})\sum_{k=1}^K\frac{|\mathcal{D}_k|}{|\mathcal{D}|}\mathbb{E}\left[\hat{\bm{g}}_k (\bm{\hat{w}}_{k}^{(t)}) \right] \\ 
        & = -\eta \bm{g}^T(\bm{w}^{(t)})\sum_{k=1}^K\frac{|\mathcal{D}_k|}{|\mathcal{D}|}\Tilde{\bm{g}}_k (\bm{w}^{(t)}) = -\eta \|\bm{g} (\bm{w}^{(t)}) \|^2.\\
      \end{aligned}
\end{equation}
For the second term on the right side of \eqref{equ:Exp}, based on Jensen's inequality, we obtain 
$\frac{L\eta^2}{2}\mathbb{E}\left[ \left\| \sum_{k=1}^K\frac{|\mathcal{D}_k|}{|\mathcal{D}| }\hat{\bm{g}}_k (\bm{\hat{w}}_{k}^{(t)}) \right\|^2\right]  \leq \frac{L\eta^2}{2}\mathbb{E}\left[ \sum_{k=1}^K\frac{|\mathcal{D}_k|}{|\mathcal{D}| } \left\| \hat{\bm{g}}_k (\bm{\hat{w}}_{k}^{(t)}) \right\|^2\right].$
By utilizing Peter-Paul inequality, it is further calculated as
\begin{equation} \label{2}
    \begin{aligned}
        & \frac{L\eta^2}{2}\mathbb{E}\left[ \left\| \sum_{k=1}^K\frac{|\mathcal{D}_k|}{|\mathcal{D}| }\hat{\bm{g}}_k (\bm{\hat{w}}_{k}^{(t)}) \right\|^2\right] \\
        & \leq \frac{L\eta^2}{2}\mathbb{E}\left[ \sum_{k=1}^K\frac{|\mathcal{D}_k|}{|\mathcal{D}| } \left\| \hat{\bm{g}}_k (\bm{\hat{w}}_{k}^{(t)}) - \tilde{\bm{g}}_k (\bm{w}^{(t)}) +  \tilde{\bm{g}}_k (\bm{w}^{(t)})\right\|^2\right]\\
         & \leq  \frac{L\eta^2}{2}\mathbb{E}\left[ 2 \sum_{k=1}^K\frac{|\mathcal{D}_k|}{|\mathcal{D}| }  \left\| \hat{\bm{g}}_k (\bm{\hat{w}}_{k}^{(t)}) - \tilde{\bm{g}}_k (\bm{w}^{(t)}) \right\|^2 \right] \\
        & \quad +  \frac{L\eta^2}{2}\mathbb{E}\left[ 2 \sum_{k=1}^K\frac{|\mathcal{D}_k|}{|\mathcal{D}| } \left\| \tilde{\bm{g}}_k (\bm{w}^{(t)}) \right\|^2 \right].\\
    \end{aligned}
\end{equation}
Then, taking variance bound gradient \eqref{equ:variance} into 
\eqref{2}, we obtain
\begin{equation} \label{3}
    \begin{aligned}
        & \frac{L\eta^2}{2}\mathbb{E}\left[ \left\| \sum_{k=1}^K\frac{|\mathcal{D}_k|}{|\mathcal{D}| }\hat{\bm{g}}_k (\bm{\hat{w}}_{k}^{(t)}) \right\|^2\right] \\
         & \leq  L\eta^2 \sum_{k=1}^K \frac{|\mathcal{D}_k|}{|\mathcal{D}|} A^2G^2 \frac{\gamma_{k,t}}{1-\gamma_{k,t}}\\
         & \quad + L\eta^2\mathbb{E}\left[ \sum_{k=1}^K\frac{|\mathcal{D}_k|}{|\mathcal{D}| } \left\| \tilde{\bm{g}}_k (\bm{w}^{(t)}) \right\|^2 \right].\\
    \end{aligned}
\end{equation}
According to Assumption \ref{assumption6} and deploying Peter-Paul inequality, \eqref{3} is derived as
\begin{equation} \label{equ:basic2}
    \begin{aligned}
            & \frac{L\eta^2}{2}\mathbb{E}\left[ \left\| \sum_{k=1}^K\frac{|\mathcal{D}_k|}{|\mathcal{D}| }\hat{\bm{g}}_k (\bm{\hat{w}}_{k}^{(t)}) \right\|^2\right] \\
            & \leq  L\eta^2 \sum_{k=1}^K \frac{|\mathcal{D}_k|}{|\mathcal{D}|} A^2G^2 \frac{\gamma_{k,t}}{1-\gamma_{k,t}} \\
            &\quad + 2L\eta^2 \mathbb{E}\left[  \sum_{k=1}^K\frac{|\mathcal{D}_k|}{|\mathcal{D}| } \left\| \bm{g} (\bm{w}^{(t)}) \right\|^2 \right]\\
            & \quad + 2L\eta^2 \mathbb{E}\left[  \sum_{k=1}^K\frac{|\mathcal{D}_k|}{|\mathcal{D}| } \left\| \tilde{\bm{g}}_k (\bm{w}^{(t)}) - \bm{g} (\bm{w}^{(t)}) \right\|^2 \right] \\           
            & \leq L\eta^2 \sum_{k=1}^K \frac{|\mathcal{D}_k|}{|\mathcal{D}|} A^2G^2 \frac{\gamma_{k,t}}{1-\gamma_{k,t}} \\
            & \quad + 2L\eta^2  \|\bm{g} (\bm{w}^{(t)}) \|^2 + 2LK\eta^2  \frac{\sigma^2}{|\mathcal{D}|}. \\
    \end{aligned}
\end{equation}
Substituting \eqref{equ:expect_1} and \eqref{equ:basic2} into \eqref{equ:Exp}, we obtain 
\begin{equation} \label{equ:basic3}
    \begin{aligned}
            & \mathbb{E} \left[\mathcal{F}\left(\bm{w}^{(t+1)}\right)-\mathcal{F}\left(\bm{w}^{(t)}\right) \right] \\
            & \leq (-\eta + 2L\eta^2) \|\bm{g} (\bm{w}^{(t)}) \|^2 + 2LK\eta^2 \frac{\sigma^2}{|\mathcal{D}|} \\
            & \quad + L\eta^2 \sum_{k=1}^K \frac{|\mathcal{D}_k|}{|\mathcal{D}|} A^2G^2 \frac{\gamma_{k,t}}{1-\gamma_{k,t}}.
    \end{aligned}
\end{equation}
Let the learning rate be $\eta=\frac{1}{3 \sqrt{T}L}$, 
we end the proof of Lemma \ref{lem:adaptive}.

\section{PROOF OF THEOREM \ref{theorem:adaptive}}\label{_proof_of_theorem:adaptive}
By rearranging \eqref{equ:lemma_opt}, we obtain 
\begin{equation} \label{equ:adap_re}
    \begin{aligned}
       & \left(\frac{1}{3\sqrt{T}L}-\frac{2}{9TL} \right)\|\bm{g}(\bm{w}^{(t)}) \|^2\\
        & \leq \mathbb{E} \left[\mathcal{F}\left(\bm{w}^{(t)}\right)-\mathcal{F}\left(\bm{w}^{(t+1)}\right) \right]\\
        & \quad + \frac{2K}{9TL} \frac{\sigma^2}{|\mathcal{D}|} + \frac{A^2G^2}{9TL} \sum_{k=1}^K \frac{|\mathcal{D}_k|}{|\mathcal{D}|} \cdot \frac{\gamma_{k,t}}{1-\gamma_{k,t}} .\\
    \end{aligned}
\end{equation}
Since $\sqrt{T} \leq T,\; T=1,2,\ldots$, the left side of \eqref{equ:adap_re} is scaled as 
$$\left(\frac{1}{3\sqrt{T}L}-\frac{2}{9\sqrt{T}L} \right)\|\bm{g}(\bm{w}^{(t)}) \|^2 = \frac{1}{9\sqrt{T}L}\|\bm{g}(\bm{w}^{(t)}) \|^2 ,$$
then, multiply $9\sqrt{T}L$ on both sides, we obtain 
\begin{equation}\label{Eq:Convergence1} 
    \begin{aligned}
        & \|\bm{g} (\bm{w}^{(t)}) \|^2 
         \leq 9\sqrt{T}L\left[\mathcal{F}\left(\bm{w}^{(t)}\right) -\mathcal{F}\left(\bm{w}^{(t+1)}\right) \right] \\
         & \quad +\frac{2K}{\sqrt{T}}   \frac{\sigma^2}{|\mathcal{D}|}   + \frac{A^2G^2}{\sqrt{T}} \sum_{k=1}^K \frac{|\mathcal{D}_k|}{|\mathcal{D}|}  \frac{\gamma_{k,t}}{1-\gamma_{k,t}}.\\
    \end{aligned}
\end{equation}
Next, average both sides across all communication rounds $t=0, 1, \ldots, T-1$ obtains 
\begin{equation} 
    \begin{aligned}
        & \frac{1}{T}\sum_{t=0}^{T-1} \|\bm{g} (\bm{w}^{(t)}) \|^2 
         \leq \frac{9L}{\sqrt{T}}  \left[\mathcal{F}\left(\bm{w}^{(0)}\right) -\mathcal{F}\left(\bm{w}^{(T)}\right) \right] \\
        &  \quad + \frac{2K}{\sqrt{T}}   \frac{\sigma^2}{|\mathcal{D}|} +  \frac{A^2G^2}{\sqrt{T}}\frac{1}{T} \sum\limits_{t=0}^{T-1}\sum_{k=1}^K \frac{|\mathcal{D}_k|}{|\mathcal{D}|} \frac{\gamma_{k,t}}{1-\gamma_{k,t}}.
    \end{aligned}
\end{equation}
As the number of rounds goes to infinity, 
we end the proof of Theorem \ref{theorem:adaptive}.

\bibliographystyle{IEEEtran}
\bibliography{refs}

\end{document}